\documentclass[letterpaper]{article} 
\usepackage[]{aaai2026}  
\usepackage{times}  
\usepackage{helvet}  
\usepackage{courier}  
\usepackage[hyphens]{url}  
\usepackage{graphicx} 
\usepackage{natbib}  
\usepackage{caption} 
\title{Dynamic Tree Databases in Automated Planning}

\author{
    Oliver Joergensen\textsuperscript{\rm 1},
    Dominik Drexler\textsuperscript{\rm 2},
    Jendrik Seipp\textsuperscript{\rm 3},
}
\affiliations{
    Link{\"o}ping University, Sweden\\
    \textsuperscript{\rm 1}oliver.harold.joergensen@liu.se,
    \textsuperscript{\rm 2}dominik.drexler@liu.se,
    \textsuperscript{\rm 3}jendrik.seipp@liu.se,
}

\usepackage{algorithm}
\usepackage[indLines=false,commentColor=black]{algpseudocodex}
\usepackage{booktabs}
\usepackage{multirow,makecell}
\usepackage{amsmath,amssymb,amsthm}
\usepackage{multicol}
\usepackage{stmaryrd}
\usepackage{mathtools}
\usepackage{subcaption}
\usepackage{tikz}
\usetikzlibrary{shapes,arrows,arrows.meta}
\usetikzlibrary{automata}
\usetikzlibrary{positioning}
\usetikzlibrary{backgrounds}
\usetikzlibrary{patterns,patterns.meta,angles}
\usepackage{xspace}

\usepackage{courier}
\usetikzlibrary{graphs, quotes}
\usepackage{forest}
\usepackage{pgfplots}
\pgfplotsset{compat=1.18}
\usepackage[mode=buildnew]{standalone}
\usepackage{paralist}
\usepackage[acronym]{glossaries}

\usepackage[main=american]{babel}
\useshorthands*{"}
\defineshorthand{"-}{\babelhyphen{hard}}
\defineshorthand{"=}{\babelhyphen{–}}
\defineshorthand{"/}{\babelhyphen{/}}
\usepackage{stmaryrd}

\newcommand{\Omit}[1]{}

\newcommand{\defined}[1]{\emph{#1}}

\newcommand{\dtreedbsswiss}{DTDB-S}
\newcommand{\dtreedbshashid}{DTDB-H}
\newcommand{\resizefactor}{\ensuremath{\delta}}

\newcommand{\treestructure}{\ensuremath{\lambda}}

\definecolor{tblue}{RGB}{0,119,187}
\definecolor{tcyan}{RGB}{51,187,238}
\definecolor{tteal}{RGB}{0,153,136}
\definecolor{torange}{RGB}{238,119,51}
\definecolor{tred}{RGB}{204,51,17}
\definecolor{tmagenta}{RGB}{238,51,119}
\definecolor{tgray}{RGB}{187,187,187}

\newtheorem{theorem}{Theorem}

\theoremstyle{definition}

\definecolor{darkgreen}{rgb}{0.0, 0.2, 0.13}
\iftrue
    
    \newcommand{\oliver}[1]{\textcolor{orange}{Oliver: #1}}
    \newcommand{\dominik}[1]{\textcolor{magenta}{Dominik: #1}}
    \newcommand{\jendrik}[1]{\textcolor{blue}{Jendrik: #1}}
    
\else
    
    \newcommand{\oliver}[1]{}
    \newcommand{\dominik}[1]{}
    \newcommand{\jendrik}[1]{}
    
\fi
\newcommand{\notes}[1]{}

\newcommand{\inlinecite}[1]{\citet{#1}}
\newcommand{\egcite}[1]{\citep[e.g.,][]{#1}}

\newcommand{\denselist}{\itemsep 0pt\partopsep 0pt}

\newcommand{\tuple}[1]{\ensuremath{\langle #1 \rangle}}

\newcommand{\datastructure}{\ensuremath{\mathcal{D}}}
\newcommand{\representation}{\ensuremath{\mathit{rep}}}
\newcommand{\bigO}{\ensuremath{\mathcal{O}}}
\newcommand{\bigTheta}{\ensuremath{\Theta}}
\newcommand{\opinsert}{\ensuremath{\mathit{insert}}}
\newcommand{\oplookup}{\ensuremath{\mathit{lookup}}}

\newcommand{\wordsize}{\ensuremath{w}}

\newcommand{\hash}{\ensuremath{h}}
\newcommand{\key}{\ensuremath{k}}
\newcommand{\keypos}{\ensuremath{i}}
\newcommand{\capacity}{\ensuremath{c}}

\newcommand{\alphabet}{\ensuremath{\Sigma}}
\newcommand{\indexedhashset}{\ensuremath{f}}
\newcommand{\vertex}{\ensuremath{v}}
\newcommand{\leftvertex}{\ensuremath{l}}
\newcommand{\rightvertex}{\ensuremath{r}}
\newcommand{\sequence}{\ensuremath{z}}
\newcommand{\sequencelength}{\ensuremath{k}}
\newcommand{\sequencecount}{\ensuremath{n}}

\newcommand{\statetransitionsystem}{\ensuremath{\mathcal{S}}}
\newcommand{\states}{\ensuremath{S}}
\newcommand{\state}{\ensuremath{s}\xspace}
\newcommand{\action}{\ensuremath{a}}
\newcommand{\initialstate}{\ensuremath{s_0}}
\newcommand{\goalstates}{\ensuremath{\states_G}}
\newcommand{\actions}{\ensuremath{A}}
\newcommand{\applicable}{\ensuremath{\mathit{App}}}
\newcommand{\successor}{\ensuremath{\mathit{Succ}}}
\newcommand{\propositionalvariable}{\ensuremath{p}}
\newcommand{\fdrvariable}{\ensuremath{v}}
\newcommand{\numericvariable}{\ensuremath{n}}
\newcommand{\mutexgroup}{\ensuremath{g}}
\newcommand{\effect}{\ensuremath{\mathit{eff}}}
\newcommand{\permutation}{\ensuremath{\pi}}
\newcommand{\gain}{\ensuremath{\mathit{gain}}}
\newcommand{\affinity}{\ensuremath{\mathit{aff}}}

\newcommand{\bin}{\ensuremath{B}}
\newcommand{\variables}{\ensuremath{\mathcal{V}}}
\newcommand{\variable}{\ensuremath{v}}

\newcommand{\none}{\ensuremath{\textit{none}}}

\newcommand{\unaligned}[1]{\noalign{\centering \ensuremath{#1}\\\kern-\baselineskip}}

\newcommand{\coverage}{\#C}
\newcommand{\outofmem}{\#M}
\newcommand{\outoftime}{\#T}
\newcommand{\memoryscore}{MS}

\algrenewcommand\algorithmicindent{1.0em}%

\usepackage{amsmath}
\usepackage{booktabs}
\usepackage{nicefrac}
\newcommand{\asep}{\hspace{15pt}}

\usepackage{tikz}
\usepackage{pgfplots}
\pgfplotsset{compat=1.17}
\usetikzlibrary{patterns}
\usetikzlibrary{arrows}
\usetikzlibrary{arrows.meta}
\usetikzlibrary{graphs,graphs.standard,calc}
\usetikzlibrary{shapes,snakes,decorations}
\usetikzlibrary{positioning}
\usetikzlibrary{matrix}
\usetikzlibrary{fit}

\usepackage{graphicx}
\usepackage{multirow}
\usepackage{microtype}

\newif\ifarxiv
\arxivtrue 

\ifarxiv
\else
  \usepgfplotslibrary{external}
\fi

\begin{document}

\maketitle

\begin{abstract}

A central challenge in scaling up explicit state-space search for large tasks is compactly representing the set of generated states.
Tree databases, a data structure from model checking, require constant space per generated state in the best case, but they need a large preallocation of memory.
We propose a novel \emph{dynamic} variant of tree databases for compressing state sets over propositional and numeric variables and prove that it maintains the desirable properties of the static counterpart.
Our empirical evaluation of state compression techniques for grounded and lifted planning on classical and numeric planning tasks reveals compression ratios of several orders of magnitude, often with negligible runtime overhead.

\end{abstract}

\section{Introduction}
Explicit state space search is a widely used technique for solving classical and numeric planning tasks \egcite{hoffmann-nebel-jair2001,richter-westphal-jair2010,lipovetzky-geffner-ecai2012,seipp-et-al-jair2020, kuroiwa-et-al-jair2022, scala-et-al-ecai2016}. It iteratively expands states, starting from the initial state, to generate their successors, and is often guided by a goal-directed heuristic to reach a goal state.
Even with strong heuristics, the set of generated states can be huge. As a result, compactly representing them is a fundamental challenge, as it enables the storage and exploration of more states within limited memory.

One approach for compressing state sets is to use \emph{tree databases} \cite{blom-et-al-entcs2008}.
They are a family of data structures that originate from model checking and support efficient state insertion and lookup operations, which are essential for explicit search.
A tree database is a forest of trees, each encoding a sequence over a finite alphabet. In automated planning, the alphabet consists of indices of propositional variables or values of numeric variables. Shared subtrees enable significant memory savings, and in the best case, the space required per state approaches the information-theoretical lower bound when successor states differ from the parent state on average in only a single variable \cite{laarman-isse2019}. These properties make tree databases particularly interesting for automated planning, where each ground action typically changes only a few state variables.

Previous work on tree databases proposed implementing them as statically sized \emph{hash ID maps} \egcite{darragh-et-al-jspe1993, blom-et-al-entcs2008, laarman-et-al-spin2011} and compact hashing \cite{cleary-ieeecomp1984}. A hash ID map is a hash table that maps keys to identifiers that are assigned by the data structure itself. The main drawback is that statically sized tree databases require manual memory management with potentially large preallocations, which are difficult to estimate in advance.

In this paper, we propose a dynamic variation of tree databases that effectively addresses these drawbacks. We also conduct an extensive empirical evaluation on classical and numeric planning benchmarks, covering both grounded and lifted settings. Specifically, we evaluate our method on two representative state encodings: the finite-domain representation used for grounded planning in Fast Downward \cite{helmert-jair2006}, and the sparse propositional and dense numeric representation used for lifted planning in Mimir \cite{stahlberg-ecai2023}. In terms of state expansion rates, both planners are state-of-the-art \cite{seipp-et-al-jair2020, drexler-2025}, making them ideal candidates for evaluating our dynamic tree databases, since they can generate a large number of states in a short time.
Our empirical evaluation reveals significant reductions in overall memory usage during uninformed search for larger instances, resulting in a higher number of solved tasks for the lifted case, with negligible runtime overhead. Moreover, we observe compression ratios of up to three orders of magnitude, demonstrating that tree databases are a powerful technique for scaling explicit state space search in automated planning to larger tasks. While similar reductions can be observed in the grounded setting, since the existing data structures already exhibit high memory efficiency, only a limited subset of instances shows reduced memory usage. In summary, our key contributions are:
\begin{enumerate} \denselist
    \item We propose \emph{dynamically}-sized tree databases for compressing sequences $\sequence$ of $\sequencelength$ elements over arbitrary finite alphabets with amortized $\bigTheta(\sequencelength)$ insertion time.
    \item We apply our dynamic tree databases to compress planning states that consist of both propositional and numeric state variables.
    \item We empirically demonstrate the effectiveness of dynamic tree databases on planning benchmark sets.
\end{enumerate}

\section{Related Work}

In this section, we present several related works on compression techniques for state space search. For this discussion, we assume a {RAM} model with $\wordsize$-bit words, and a set of $\sequencecount$ states, each consisting of $\sequencelength$ elements from a finite universe.

\paragraph{Optimal Compression of Combinatorial State Spaces.} \inlinecite{laarman-isse2019} analyzes the optimal compression ratio of combinatorial state spaces from an information-theoretical perspective. More precisely, they seek a lower bound on the number of bits required to encode a combinatorial state space.
They assume that on average only a single variable changes, which is uniformly picked, resulting in a lower bound of a single $\wordsize$-bit word plus $\bigO(\log_2(\sequencelength))$ bits per state.
This result also applies to classical planning, where ground actions usually modify only a small number of state variables.

\paragraph{Finite-Domain Representation (FDR).} A set of propositional state variables forms a \emph{mutex group}, if at most one of them is true in every reachable state \cite{helmert-aij2009, fiser-aaai2020}.
A mutex group $\mutexgroup$ with $m = |\mutexgroup|$ propositions can be represented compactly by introducing a single \emph{finite-domain variable} with $m$ values, encoded as a $\wordsize$-bit word.
We call the representation that uses one $\wordsize$-bit word per finite-domain variable the \emph{unpacked} FDR representation.
As an additional compression, each value can be encoded in $\lceil \log_2(m)\rceil$ bits, rather than $\wordsize$ bits. The \emph{packed} FDR representation of a state $\state$ concatenates the binary representations of the values of every finite-domain variable in $\state$. Its size scales polynomially in the number of variables, which can become a limiting factor.

\paragraph{Binary Decision Diagrams (BDDs).}
Binary Decision Diagrams (BDDs) are canonical graph-based representations of Boolean functions \cite{bryant-dac1985}. BDDs are represented as directed acyclic graphs. In symbolic search \cite{burch-et-al-iandc1992,edelkamp-kissmann-aaai2008,torralba-et-al-aij2017}, BDDs symbolically represent the successor function. Applying the successor function yields sets of states that are also symbolically represented as BDDs. The size of a BDD depends on the ordering of its variables. In the best case, it can represent exponentially many states in space that is polynomial in the number of variables. Finding an optimal variable ordering that minimizes the size of a BDD is NP-complete \cite{bollig-wegener-ieeecomp1996}.
In explicit search, any state needs to be uniquely indexable. BDDs represent a single state as a terminating path from the root note. As there may be multiple such paths, one cannot uniquely identify a state based on a node index, making BDDs inappropriate for explicit search.

\paragraph{Level-Ordered Edge Sequences (LOES).}
Unlike BDDs, LOES avoids the use of pointers by encoding the binary decision tree as a flat edge sequence ordered by tree depth, leading to a more compact memory layout \cite{schmidt-zhou-aaai2011}. Similar to BDDs, its compression efficiency is sensitive to the ordering of state variables. Efficient insertion into LOES requires batching, which amortizes the cost of updating the structure. For a set of $\sequencecount$ states of size $\sequencelength$, LOES requires at most twice the size of the concatenated packed FDR representation in the worst case and as few as $2\sequencecount$ bits in the best case when prefix sharing is maximal, when $\sequencecount>\sequencelength$.
LOES, however, is unsuitable for tasks that incrementally add states, such as explicit search, because of overly expensive and frequent resizing operations.

\paragraph{Decoupled Search.}
Decoupled search \cite{gnad-hoffmann-aij2018} exploits the task structure by partitioning the variables into a central component and leaf components, wherein the leaf components depend solely on the central component.
One aspect that makes decoupled search effective is its compact state representation, achieved by combining the state of the central component with a reachability function over the leaf states.
These two jointly represent a set of states.
The function is a mapping that indicates whether a leaf state is reachable from the initial state via the current central state, and is implemented practically as a set of unique \emph{IDs} that map to leaf states.
In the best case, the state size reduces to a polynomial in the number of leaf factors and central component variable.
However, in the worst case, the reachable decoupled state space can be exponentially larger than the explicit state space.

\paragraph{Tree Databases.} Tree databases are forests of binary trees, each encoding an input sequence over an arbitrary finite alphabet \cite{blom-et-al-entcs2008}. See Figure~\ref{fig:tree-database} for an example. In essence, tree databases subdivide each input sequence until the leaves contain pairs of elements, while inner nodes are pairs that point to their left and right children. The trees in the forest may share common subtrees, reducing memory usage. Tree databases can support parallelization \cite{blom-et-al-entcs2008} and variable-length input sequences \cite{freark-nfm2021}. The worst-case representation size for a state is $2\sequencelength\wordsize - 2\wordsize$, and the optimal size approaches $2\wordsize$, which is constant \cite{blom-et-al-entcs2008}. In the best case, the space usage of tree databases approaches the information-theoretical lower bound when only a single variable changes between parent and successor state \cite{laarman-isse2019}. Tree databases are often implemented as hash ID maps \cite{darragh-et-al-jspe1993} and combined with compact hashing \cite{cleary-ieeecomp1984} for space efficiency.

\begin{figure}[t]
    \centering
    \begin{subfigure}{0.47\textwidth}
        \centering
        \includegraphics[width=0.67\textwidth]{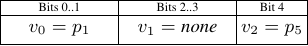}
        \caption{Packed FDR representation of state $\state$ in which only propositions $\propositionalvariable_1$ and $\propositionalvariable_5$ hold. Assume that the task has mutex groups $\{\propositionalvariable_0, \propositionalvariable_1, \propositionalvariable_2\}$, $\{\propositionalvariable_3, \propositionalvariable_4\}$, and $\{\propositionalvariable_5\}$, which give rise to the FDR variables $\fdrvariable_0$, $\fdrvariable_1$, and $\fdrvariable_2$, with domains $\{\propositionalvariable_0, \propositionalvariable_1, \propositionalvariable_2, \none\}$, $\{\propositionalvariable_3, \propositionalvariable_4, \none\}$, and $\{\propositionalvariable_5, \none\}$, respectively. (The $\none$ value indicates that none of the propositions in the mutex group holds.) 
 It sequentially encodes the values of $\fdrvariable_0$, $\fdrvariable_1$, and $\fdrvariable_2$ with 2, 2, and 1 bits, respectively, resulting in a total of 5 bits.}
        \label{fig:dense-fdr-representation}
    \end{subfigure}

    \vspace{0.2cm}

    \begin{subfigure}{0.47\textwidth}
        \centering
        \includegraphics[width=0.67\textwidth]{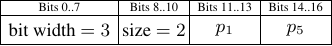}
        \caption{Sparse propositional representation of state $\state$, in which only propositions $\propositionalvariable_1$ and $\propositionalvariable_5$ hold. The first field uses one byte to encode the maximum number of bits necessary to store the state length and atom indices ($\lceil \log_2(\max(2, 1, 5))\rceil = 3$), followed by the number of propositions, and the indices of the two propositions, resulting in a total of 17 bits.}
        \label{fig:sparse-propositional-representation}
    \end{subfigure}

    \vspace{0.2cm}

    \begin{subfigure}{0.47\textwidth}
        \centering
        \includegraphics[width=0.47\textwidth]{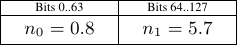}
        \caption{Dense numeric representation of $\numericvariable_0 = 0.8$ and $\numericvariable_1=5.7$. It sequentially encodes the values of every numeric variable using binary64 encoding, with 64 bits per variable, resulting in a total of 128 bits.}
        \label{fig:dense-numeric-representation}
    \end{subfigure}

    \caption{(Partial) state representations for a task with propositions $\propositionalvariable_0$, \dots, $\propositionalvariable_5$ and numeric variables $\numericvariable_0$ and $\numericvariable_1$.}
    \label{fig:rep-example}
\end{figure}

\section{Background}

In this section, we present the background of classical and numeric planning, state set compression, and tree databases.

\subsection{Classical and Numeric Planning}

A (numeric) \defined{planning task} is typically specified in a compact, factored representation, such as a first-order description in the Planning Domain Definition Language (PDDL) \cite{mcdermott-et-al-tr1998,fox-long-jair2003} or a grounded $\text{SAS}^+$ encoding \cite{backstrom-nebel-compint1995}. If a task contains only propositional variables (i.e., no numeric state variables), it is called a \emph{classical planning task}.
Each planning task induces a (possibly infinite) state-space model that captures its semantics. In this work, we operate at this semantic level and do not rely on the internal syntactic structure of each action. A \defined{state space model} is a tuple $\statetransitionsystem = \langle \states, \initialstate, \allowbreak \goalstates, \actions, \allowbreak \applicable, \allowbreak \successor \rangle$, where $\states$ is the set of states, in which each state is a value assignment to a set of propositional and numeric variables, $\initialstate \in \states$ is the initial state, $\goalstates \subseteq \states$ is the set of goal states, $\actions$ is a set of actions, $\applicable$ is a function that maps each state $\state$ to a set of actions from $\actions$ that are applicable in $\state$, $\successor$ is a function that maps each state $\state$ and action $\action$ in $\applicable(\state)$ to the deterministic successor state $\state'$ A \defined{state trajectory} is a sequence $\state_1,\state_2,\dots,\state_n$ such that for all $i=1,\dots,n-1$ there exists an action $\action\in\applicable(\state_i)$ with $\successor(\state_i,\action) = \state_{i+1}$. A state trajectory that starts in $\state_0$ and ends in a goal state is a \defined{plan}. We can find a plan by running a search, i.e., starting from the initial state, we iteratively generate the applicable actions and their successors until reaching a goal state.

We make the closed-world assumption, meaning that atoms that are not true in a state are implicitly assumed to be false. This leads to what we call the \defined{sparse propositional representation}, where a state is represented as the set of true atoms.
In planning, we distinguish between two different search modes.
First, in \defined{lifted planning}, we compute the applicable actions $\applicable(\state)$ from a given lifted first-order logic representation of the actions on the fly \cite{correa-et-al-icaps2020,stahlberg-ecai2023}.
Second, in \defined{grounded planning}, we translate the lifted representation into a grounded representation, where each atom corresponds to a Boolean state variable. In addition, we generate mutex groups over these variables to form finite-domain variables, also known as \defined{FDR variables}, where each represents the Boolean variable in its corresponding mutex group compactly \cite{helmert-aij2009}.
Due to a lack of a practically reasonable default value for numeric variables, we assume a \defined{dense numeric representation}, where each numeric variable uses 64 bits to represent its value.
Figure~\ref{fig:rep-example} illustrates the different variable representations.

\subsection{State Set Data Structure}

An explicit search algorithm requires a \emph{state set} data structure $\datastructure$ for representing the set of currently generated states, along with an $\opinsert(\state)$ operation that inserts a newly generated state $\state$ into $\datastructure$ and returns an index $i$, and a $\oplookup(i)$ operation that returns the state with index $i$. We say an insert operation is \emph{efficient} if it runs in amortized $\bigO(|\state|)$ time, where $|\state|$ denotes the number of represented variables in $\state$. Similarly, we call a lookup operation \emph{efficient} if it runs in $\bigO(|\state|)$ time.

Beyond runtime efficiency, we are also interested in the space efficiency of a state set data structure.
For a given set of states $\states'\subseteq \states$, a state set data structure $\datastructure$ maps $\states'$ via a representation function $\representation_\datastructure$ to a binary representation $\{0,1\}^*$ with size $|\representation_\datastructure(\states')|$.
The compression ratio for a state set $\states'$ is the quotient between the size of the uncompressed ($\datastructure$) and the compressed ($\datastructure'$) representation, i.e., $|\representation_\datastructure(S')| / |\representation_{\datastructure'}(\states')|$.

\subsection{Tree Databases}

\begin{figure*}[tb]
    \centering
    \includegraphics[width=0.7\textwidth]{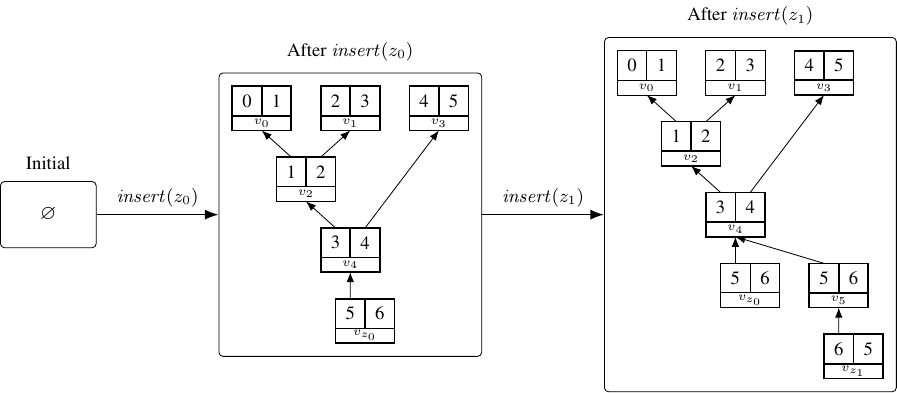}
    \caption{A tree database of perfectly balanced trees after the insertion of two sequences: $\sequence_0 = \tuple{0,1,2,3,4,5}$ and $\sequence_1 = \tuple{1,2,4,5,6}$ with respective lengths $\sequencelength_{\sequence_0} = 6$ and $\sequencelength_{\sequence_1} = 5$. We depict each node $\vertex_i$ with left and right children/values $\leftvertex(\vertex)$ and $\rightvertex(\vertex)$ at the top, and index $i=\indexedhashset(\leftvertex(\vertex), \rightvertex(\vertex))$ at the bottom. Moreover, we use boldface to denote (re-)inserted nodes. Next, we explain the tree with root node $\vertex_{\sequence_0}$ for $\sequence_0$. Since $\treestructure(\vertex_{\sequence_0}) = \lfloor(\sequencelength_{\sequence_0} + 1) / 2\rfloor = 3$, we have that $\treestructure(\vertex_2) =  2^{\lfloor \log_2(\treestructure(\vertex_{\sequence_0}) - 1) \rfloor} = 2$ and $\treestructure(\vertex_3) = \treestructure(\vertex_2)-\treestructure(\vertex_{\sequence_0}) = 1$. Next, since $\treestructure(\vertex_2) = 2$, we have that $\treestructure(\vertex_0) = 2^{\lfloor \log_2(\treestructure(\vertex_2) - 1) \rfloor} = 1$ and $\treestructure(\vertex_1) = \treestructure(\vertex_2) - \treestructure(\vertex_0) = 1$. Hence, we know that $\vertex_0, \vertex_1$, and $\vertex_3$ are the leaf nodes containing the values in order from left to right. Finally, we explain the tree with root node $\vertex_{\sequence_1}$. First, observe that the tree for $\sequence_1$ shares nodes with the tree for $\sequence_0$. More precisely, the inner node $\vertex_2$ (resp.~inner node $\vertex_4$, leaf node $\vertex_3$) in the tree for $\sequence_0$ is a leaf node (resp.~inner node, leaf node) in the tree for $\sequence_1$. Second, observe that the last integer element $6$ in $\sequence_1$ is stored directly in the right entry of inner node $\vertex_5$, i.e., $\rightvertex(\vertex_5) = 6$.}
    \label{fig:tree-database}
\end{figure*}

Ground actions in planning tasks often change only a few state variables. Consequently, a state and its successor state often share many ground atoms and numeric variable value assignments. Tree databases encode each state as a tree where each node recursively subdivides the sequence of ground atoms and numeric variable assignments. Moreover, tree databases store trees in a forest, allowing shared subtrees across states, which can result in significant space savings compared to representations that do not exploit shared structure. 

A tree database over a finite alphabet $\alphabet$ is a forest of binary trees. A \defined{tree} in a tree database over $\alphabet$ consists of nodes of the form $\vertex = \tuple{\leftvertex(\vertex),\rightvertex(\vertex)}$, where $\leftvertex(\vertex)$ and $\rightvertex(\vertex)$ are either child nodes or elements in $\alphabet$. The structure of a tree is defined by a function $\treestructure$, which maps each node $\vertex$ to the number of leaf nodes in the subtree rooted at $\vertex$. If $\treestructure(\vertex) = 1$, then $\vertex$ is a leaf node, and $\leftvertex(\vertex)$ and $\rightvertex(\vertex)$ are elements in $\alphabet$. If $\treestructure(\vertex)>1$, then the tree is \defined{perfectly balanced} iff $\treestructure(\leftvertex(\vertex)) = 2^{\lfloor \log_2(\treestructure(\vertex) - 1) \rfloor}$, $\treestructure(\rightvertex(\vertex)) = \treestructure(\vertex) - \treestructure(\leftvertex(\vertex))$.\footnote{\inlinecite{blom-et-al-entcs2008} originally defined trees as \emph{balanced} iff the number of leaf nodes at the left subtree $\leftvertex(\vertex)$ is $\treestructure(\leftvertex(\vertex)) = \lceil \treestructure(\vertex) / 2\rceil$, which is appropriate for fixed-length sequences. The notion of \emph{perfectly balanced} is beneficial for allowing variable-length sequences, as it restricts left subtrees to have a total number of leaf nodes that is a power of two. Consequently, $\lambda$ in perfectly balanced trees often creates structurally similar left subtrees, which is a prerequisite for effective node sharing among several trees consisting of varying numbers of nodes. \cite{freark-nfm2021}.}

A tree database unambiguously encodes a sequence $\sequence = \sequence_0,\sequence_1,\ldots,\sequence_{\sequencelength-1}$ consisting of $\sequencelength$ elements over $\alphabet$ in a perfectly balanced tree with root node $\vertex$ and $\treestructure(\vertex) = \lfloor (\sequencelength + 1) / 2 \rfloor$ leaf nodes such that the $i$-th leaf node is $\vertex_i = \tuple{\sequence_{2i},z_{2i+1}}$ with $i=0,\ldots,\treestructure(\vertex)-1$. If $\sequencelength$ is odd, $\sequence_{\sequencelength-1}$ has no partner, and hence, we store $\sequence_{\sequencelength-1}$ directly in $\rightvertex(\vertex')$ of its parent $\vertex'$ and remove the leaf $\vertex_{\treestructure(\vertex) - 1}$. In a tree database, there are operations for inserting and looking up a sequence. Figure~\ref{fig:tree-database} illustrates a sequence of insert operations into a tree database with perfectly balanced trees, initially empty.

Table \ref{tab:space_reqiurements_theory} shows theoretical bounds by \inlinecite{laarman-isse2019} regarding the space requirements of tree databases and conventional hash tables to represent a set of sequences of length $\sequencelength$ over $\wordsize$-bit integers. In the worst case, the tree database requires approximately twice as much space, while in the best case, the required space approaches that of two integers. Therefore, the best-case compression ratio is independent of the actual length $\sequencelength$ of a sequence.

\begin{table}[tb]
    \centering
    \begin{tabular}{@{}lcc@{}}
 Structure & Worst case & Best case \\
        \midrule
 Hash table & $\sequencelength\wordsize$ & $\sequencelength\wordsize$ \\
 Tree database & $2\sequencelength\wordsize-2\wordsize$ & $2\wordsize+\epsilon \wordsize$ \\
    \end{tabular}
    \caption{Theoretical bounds by \inlinecite{laarman-isse2019} on the representation size of a sequence of size $\sequencelength$ over $\wordsize$-bit integers.}
    \label{tab:space_reqiurements_theory}
\end{table}

\paragraph{Implementation Details.}

A key component of tree database implementations is an \defined{indexed hash set} representing a bijective function $\indexedhashset$ that maps keys to indices with operations to insert a key and look up a key given its index. An indexed hash set for node keys $\tuple{\leftvertex(\vertex),\rightvertex(\vertex)}$ promotes the reuse of existing nodes in different trees. A second indexed hash set with keys $\tuple{\vertex, \sequencelength}$ where $\vertex$ is a root node and $\sequencelength$ is the respective sequence length encodes the tree structure $\treestructure$ for lookup. A third indexed hash set with keys from $\alphabet$ can deduplicate elements from $\alphabet$. We say that insertion and lookup of a sequence of length $\sequencelength$ into a forest of perfectly balanced trees is \emph{efficient} if it requires amortized $\bigO(\sequencelength)$ time. For a given key $\tuple{\vertex, \sequencelength}$, depth-first traversal using the tree structure induced by $\treestructure$ enables efficient lookup \cite{blom-et-al-entcs2008}. Insertion follows the same approach, but requires further analysis when dynamically resizing an indexed hash set.

\section{Dynamic Tree Databases}

We now describe our implementations of \emph{dynamic} tree databases based on Google's \emph{flat hash tables}, also known as Swiss tables \cite{benzaquen-et-al-2018}.

\subsection{Motivation}

\inlinecite{blom-et-al-entcs2008} proposed tree databases as statically sized hash ID maps, where the position of a node implicitly becomes its identifier. Estimating a reasonable size in advance is challenging when other parts of the system additionally allocate memory. For example, a static capacity for $2^{32}$ nodes of $8$ bytes each already results in $32$ GiB of memory usage. Therefore, we propose \emph{dynamic} tree databases that overcome this limitation by dynamically resizing when the database reaches its maximum capacity. Our dynamic tree databases build upon modern flat hash tables for hash maps.

\subsection{Flat Hash Tables}

Flat hash tables are open-addressing hash tables, meaning that they store keys compactly in a dynamic array of size $\capacity$, with a maximum load factor of $7/8$, and a hash function $\hash$. In addition to the keys, flat hash tables maintain a parallel array of control bytes, one per key, which store the upper $7$ bits of the hash value. Special control byte values (sentinels) are reserved to indicate empty or deleted slots. Hence, each table entry consists of a key and a control byte, resulting in an overhead of one byte per key.

Another key innovation of flat hash tables is the use of control bytes to filter potential key matches. Each control byte act as a hash fingerprint, allowing the table to identify likely matches without accessing the actual keys. This overapproximation has a low false positive rate ($\frac{1}{2^7}$, less than 1\%) and enables the use of SIMD instructions to compare $16$ control bytes in parallel.

When looking up a key $\key$, the algorithm begins at index $\hash(\key) \bmod \capacity$ and scans forward through the control bytes.
Only if the control byte matches, the actual key is compared, significantly reducing the number of expensive key comparisons and mitigating cache misses. To insert a key $\key$, the algorithm probes forward from $\hash(\key) \bmod \capacity$ to find the first empty control byte. It then stores $\key$ in the key array and its $7$-bit fingerprint in the control byte array at the same index.

\subsection{Dynamic Stable Tree Databases (\dtreedbsswiss{})}

Recall that a tree database consists of indexed hash sets that represent a bijective function $\indexedhashset$ from keys to indices. For the backward mapping $\indexedhashset^{-1}$, we use a dynamic array in which the $\keypos$-th position stores the node $\keypos = \tuple{\leftvertex(\keypos), \rightvertex(\keypos)}$. For the forward mapping $\indexedhashset$, we use a flat hash table, which references nodes in the array using an integer position $\keypos$. Inserting a new node appends it to the end of the array, if it does not already exist, with deduplication handled by the flat hash table. This indirection is cheap, as the probability of the control bytes falsely predicting matching keys is less than 1 percent. Assuming a maximum load factor of $1.0$ for the array and $7 / 8$ for the flat has table, the dynamic nature of our implementation avoids excessive unused space. Each stored node requires one $\wordsize$-bit word and one byte of overhead in the flat hash table. When the dynamic array or the flat hash table reaches its maximum load, we increase its size by a constant factor $\resizefactor > 1$.
Due to amortized $\bigO(1)$ insertion cost into dynamic arrays using a constant resize factor $\resizefactor > 1$, \dtreedbsswiss{} is efficient.

\begin{theorem}
 Inserting a sequence $\sequence$ of length $\sequencelength$ into \dtreedbsswiss{} runs in amortized $\bigTheta(\sequencelength)$ time.
\end{theorem}

\begin{proof}
    A tree encoding a sequence $\sequence$ of length $\sequencelength$ contains $2\lfloor \sequencelength / 2 \rfloor - 1$ nodes, which is in $\bigTheta(\sequencelength)$. Each node insertion into the dynamic array and the flat hash table takes amortized $\bigTheta(1)$ time, given the constant resize factor $\resizefactor > 1$. Hence, inserting all nodes of the tree requires amortized $\bigTheta(\sequencelength)$ time.
\end{proof}

\subsection{Dynamic HashID Tree Databases (\dtreedbshashid{})}

In contrast to \dtreedbsswiss{}, \dtreedbshashid{} does not use an additional flat hash table, but instead, uses the hash location of a node $\tuple{\leftvertex(\keypos), \rightvertex(\keypos)}$ in the dynamic array as its corresponding index $\keypos$. Therefore, node indices may change upon resizing the hash table, which is triggered when the load factor reaches its maximum of $7/8$. We implemented resizing by rebuilding the trees in an array whose size increases by a constant factor $\resizefactor > 1$. The procedure employs a depth-first traversal that relocates tries in hash ID maps \cite{kanda-et-al-jea2019}, which we adapted for use with perfectly balanced trees. More specifically, it traverses each tree using the structural function $\treestructure$ and rebuilds the trees. The traversal requires a \dtreedbsswiss{} to store the root node index and its tree size, allowing us to return a stable node index while updating the unstable root node index upon resizing.

\begin{theorem}
    Inserting a sequence $\sequence$ of length $\sequencelength$ into \dtreedbshashid{} runs in amortized $\bigTheta(\sequencelength)$ time.
    \end{theorem}
\begin{proof}
    In the worst case, assume a sequence of insertions where no nodes are shared. Then, each of the $\sequencecount$ inserted sequences of length $\sequencelength$ introduces $\bigTheta(\sequencelength)$ new nodes, for a total of $\bigTheta(\sequencecount\sequencelength)$ nodes. When the hash table resizes (e.g., when the load factor exceeds the load factor of $7 / 8$), it must rehash and rebuild all existing trees, incurring a cost of $\bigTheta(\sequencecount\sequencelength)$ for the depth-first copy. However, since the table size grows geometrically (by a factor $\resizefactor > 1$), each resize accommodates $\bigTheta(\sequencecount)$ additional sequences. Therefore, the amortized cost per inserted sequence remains $\bigTheta(\sequencelength)$. Compression may delay resizes, and while a rehash could change compression, such worst-case effects do not affect the amortized bound. To illustrate, the node $\vertex_3$ in Figure~\ref{fig:tree-database} is a leaf for $\sequence_1$ but also an inner node for $\sequence_0$, which might not be after resizing, when node positions change. 
\end{proof}

\subsection{Variable Ordering}

The number of nodes in a tree database depends on the order of elements from $\alphabet$ in a sequence $\sequence$, in our case, ground atoms, FDR, or numeric variables. 
Hence, finding a variable order that minimizes the number of nodes can further decrease the memory footprint. 
Intuitively, this follows from \citet{laarman-isse2019}, who demonstrate that tree compression is optimal when successive states differ in only one variable.
An ordering that localizes these state deltas to the smallest possible subtree guaranties that the truncated tree, ending at the deepest level where only one node is updated, is optimally compressed.

We define $\effect(\action)$ as the set of ground atoms, FDR, or numeric variables affected by an action. If the elements in $\effect(\action)$ appear close together in the sequence, then the trees for a state $\state$ and its successor $\state'$ share long common suffixes, yielding many shared nodes. Conversely, if the affected elements are far apart, more internal nodes differ, and fewer nodes are shared.

A variable order is a function $\permutation : \variables\rightarrow\{0,\ldots,|\variables|-1\}$ that is also bijective. A reasonable state-independent objective is to order the variables to minimize the number of leaf nodes affected by ground actions.

\[\min\sum_{\action\in\actions} |\{\lfloor\pi(\variable) / 2\rfloor\mid \variable\in\effect(\action)\}| \]

In lifted planning, computing variable orders is more challenging as all reachable ground atoms and actions are unknown in advance. Therefore, we focus on variable ordering in the FDR representation for grounded planning, where more than two variables fit into each node, with the exact number depending on the domains of the variables in that node. Essentially, this problem can be cast to the NP-complete bin packing problem with two competing objectives: 1) minimizing the number of bins, and 2) minimizing the sum of bins affected by all operators, analogous to the objective above. Therefore, we greedily address this problem by slightly changing the objective of the commonly used bin packing strategy used for the FDR representation that works as follows: Iteratively create a new bin $\bin$ of capacity $c$ if there are unpacked variables, followed by iteratively inserting unpacked variables into $\bin$ with the largest domain size until no variable fits. Instead, we choose the variable that maximizes the $\gain$, defined as follows:

\[
\gain(\variable) = 
\begin{cases}
    \sum_{\variable'\in \variables} \affinity(\variable, \variable') & \text{ if } \bin = \emptyset \\
    \sum_{\variable'\in \bin} \affinity(\variable, \variable') & \text{ if } \bin\neq\emptyset
\end{cases}
\]

where $\affinity$ is the affinity for a pair of variables defined as $\affinity(\variable,\variable') = \sum_{\action\in\actions} \mathbf{1}\{\variable\in\effect(\action),\variable'\in\effect(\action)\}$. Intuitively speaking, $\affinity$ captures how often two variables $\variable,\variable'$ appear together in the effect of any action. Moreover, $\gain$ captures the number of co-occurences of a variable $\variable$ in the effect with those already inside the bin.

\section{Experimental Setup}

We use Fast Downward \cite{helmert-jair2006} to evaluate the compression of the FDR state representation for grounded planning, and Mimir \cite{stahlberg-ecai2023} to evaluate the compression of the sparse state representation and numeric planning, with a focus on lifted planning. 

We incorporate FDR variables into tree databases by partitioning the FDR representation into $\wordsize$-bit words such that no finite-domain variable crosses a word boundary. Moreover, we encode numeric variables in the binary64 format, commonly referred to as $8$-byte double, and store such leaf nodes in a separate \dtreedbsswiss{}, while sharing inner nodes among propositional and numeric variables.

For all planner runs, we use A$^*$ search with the blind heuristic for ease of comparison, a resize factor $\resizefactor = 2$ for dynamic arrays, and a word size $\wordsize$ equal to $32$ bits. For Fast Downward, we also run with the saturated cost partitioning heuristic using a state-of-the-art configuration \cite{seipp-et-al-jair2020}.

We limit each planner run to $5$\,hours and $8$\,GiB of memory and conduct all experiments using the Lab toolkit \cite{seipp-et-al-zenodo2017} on a compute cluster equipped with Intel Xeon Gold 6130 CPUs.

The Fast Downward configurations are as follows:
\begin{itemize} \denselist
    \item \textbf{Hashset-Unpacked} uses the default Fast Downward hashset with a $32$-bit integer for each FDR variable.
    \item \textbf{Hashset-Packed} uses the default Fast Downward hashset with the packed FDR state representation.
    \item \textbf{\dtreedbsswiss{}} uses a \dtreedbsswiss{} and bit-packs multiple FDR variables into a $32$-bit leaf node entry.
    \item \textbf{\dtreedbshashid{}} uses a \dtreedbshashid{} instead.
    \item \textbf{Aff-\dtreedbsswiss{}} is \textbf{\dtreedbsswiss{}} with variable ordering based on the affinity objective.
    \item \textbf{Aff-\dtreedbshashid{}} is \textbf{\dtreedbshashid{}} with variable ordering based on the affinity objective.
\end{itemize}
The Mimir configurations are as follows:
\begin{itemize} \denselist
    \item \textbf{Grounded-Hashset} uses a hashset and the sparse propositional and dense numeric state representation.
    \item \textbf{Grounded-\dtreedbsswiss{}} uses a \dtreedbsswiss{} with perfectly balanced trees and a separate indexed hash set to represent numeric variable assignments as $8$-byte doubles.
    \item \textbf{Grounded-\dtreedbshashid{}} uses a \dtreedbshashid{} instead.
    \item \textbf{Lifted-Hashset} is like \textbf{Grounded-Hashset} but lifted.
    \item \textbf{Lifted-\dtreedbsswiss{}} is like \textbf{Grounded-\dtreedbsswiss{}} but lifted.
    \item \textbf{Lifted-\dtreedbshashid{}} is like \textbf{Grounded-\dtreedbshashid{}} but lifted.
\end{itemize}

The supplementary material contains the code for our tree database implementations. We will make all our code available upon publication.

\subsection{Benchmarks}

We consider 6495 classical planning tasks from 166 domains. 
This benchmark set contains all 1847 STRIPS \cite{fikes-nilsson-aij1971} and all 1011 ADL \cite{pednault-kr1989} tasks from the Optimal Tracks of the International Planning Competition (IPC) from 1998 to 2023,\footnote{\url{https://github.com/aibasel/downward-benchmarks}} the 900 STRIPS tasks from the learning track of the IPC 2023, the 992 tasks for transforming quantum circuits into CNOT-only layouts \cite{shaik-vandepol-ecai2024}, the 269 tasks from the Airbus Beluga challenge in 2024,\footnote{\url{https://github.com/TUPLES-Trustworthy-AI/Beluga-AI-Challenge}} the 1058 hard-to-ground tasks (some of which use conditional effects),\footnote{\url{https://github.com/abcorrea/htg-domains}} the 223 ADL tasks from the Pushworld domain which aim at representing reasoning tasks with differently shaped objects in a grid world,\footnote{\url{https://github.com/google-deepmind/pushworld}} and the 195 object-centric ADL tasks from the Minecraft domain that result in millions of ground atoms \cite{hill-et-al-icaps2024}.
Furthermore, we consider 575 numeric planning tasks from 36 domains. 
The benchmark set contains the 380 tasks with numeric variables from the Numeric Planning Track of the IPC 2023\footnote{\url{https://github.com/ipc2023-numeric/ipc2023-dataset}}, and the 195 object-centric numeric tasks from the Minecraft domain \cite{hill-et-al-icaps2024}.

\subsection{Performance Metrics}

For a planner configuration, we aggregate the results across all benchmark tasks with the following metrics:
\begin{itemize} \denselist
    \item \textbf{\coverage} is \emph{coverage}, i.e., the number of solved tasks.
    \item \textbf{\outofmem} is the number of tasks that run out of memory.
    \item \textbf{\outoftime} is the number of tasks that run out of time.
    \item \textbf{\memoryscore} is the sum of \emph{memory scores}.
\end{itemize}
For a planner run, the \emph{memory score} is 0 if the planner runs out of time or exceeds the 8\,GiB memory limit.
The memory score is 1 for planner runs that use at most 2\,MiB.
In between those two extremes, we interpolate the memory score logarithmically:
$\text{\memoryscore}(m) = 1-\nicefrac{\ln(m)}{\ln(U)}$, where $m$ is the memory usage of the state set and $U$ is the overall memory limit (8\,GiB). The memory score provides more granular insight into memory usage than \outofmem.


\begin{table}[t]
    \setlength{\cmidrulekern}{15pt}
    \centering
    \begin{tabular}{@{}ll@{\asep}rrrr@{}}
        & Domain         & \coverage & \outofmem & \outoftime & \memoryscore \\
        \midrule
        {\multirow{6}{*}{\rotatebox[origin=c]{90}{Blind}}}
        & Hashset-Unpacked           & 1972           & 3705  & 0   & 661 \\
        & Hashset-Packed             & 2229           & 3393  & 25  & \textbf{782} \\
        & \dtreedbsswiss{}           & \textbf{2230}  & 3389  & 27  & 772 \\
        & \dtreedbshashid{}          & 2168           & 3472  & 17  & 753 \\
        & Aff-\dtreedbsswiss{}       & 2173           & 3453  & 24  & 760 \\
        & Aff-\dtreedbshashid{}      & 2111           & 3527  & 10  & 740 \\
        \midrule
        {\multirow{6}{*}{\rotatebox[origin=c]{90}{SCP}}}
        & Hashset-Unpacked           & 2455           & 2708  & 0   & 904 \\
        & Hashset-Packed             & \textbf{2659}  & 2504  & 7   & \textbf{980} \\
        & \dtreedbsswiss{}           & 2620           & 2535  & 12  & 964 \\
        & \dtreedbshashid{}          & 2575           & 2604  & 9   & 951 \\
        & Aff-\dtreedbsswiss{}       & 2627           & 2524  & 18  & 964 \\
        & Aff-\dtreedbshashid{}      & 2581           & 2575  & 13  & 951 \\
        \bottomrule
    \end{tabular}
    \caption{Results for the Fast Downward configurations on the 6495 classical planning tasks. We highlight the configuration with the highest coverage and memory score in boldface.}
    \label{tab:fast-downward-results-overview}
\end{table}


\begin{table}[t]
    \centering
    \begin{tabular}{@{}lccccc@{}}
    \textbf{Quantiles} & Min & Q1 & Median & Q3 & Max \\
    \midrule
    \textbf{Bytes}     & 4   & 4  & 8      & 24  & 13828 \\
    \end{tabular}
    \caption{Distribution of bytes required for the FDR state representation aggregated over all classical planning tasks.
    }
    \label{tab:fdr-variable_count}
\end{table}

\begin{figure*}[tbp]
    \begin{tabular}{ccc}
        \includegraphics[width=0.31\textwidth]{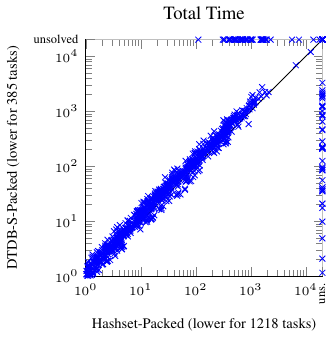} &
        \includegraphics[width=0.31\textwidth]{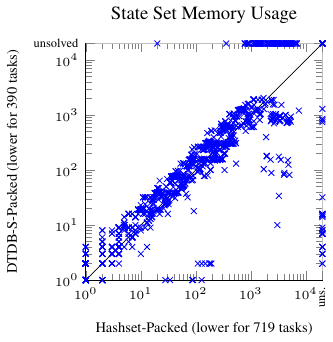} &
        \includegraphics[width=0.31\textwidth]{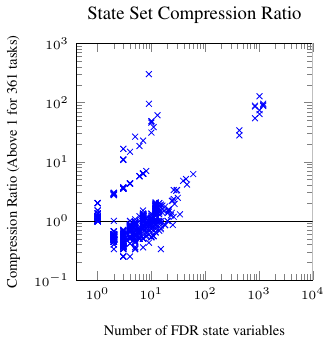}
    \end{tabular}
    \caption{Comparison of Hashset-Packed and \dtreedbsswiss{} (ours) showing the total time in seconds to solve a task (left), the memory usage for the state set in MiB (center) and the average number of state variables with the compression ratio between the configurations as described in Section 3.2 (right).}
    \label{fig:fast-downward-packed-comparison}
\end{figure*}

\section{Experimental Results for Fast Downward}

Table~\ref{tab:fast-downward-results-overview} shows that Hashset-Packed using SCP achieves the highest coverage of 2659 tasks, while the best tree database configuration Aff-\dtreedbsswiss{} achieves a lower coverage of 2627. In the blind search configuration, \dtreedbsswiss{} achieves only a coverage increase of one compared to Hashset-Packed. Most importantly, Hashset-Packed achieves the best memory score in both search configurations, demonstrating its general effectiveness across benchmark sets. 

Table~\ref{tab:fdr-variable_count} shows the number of bytes required by the FDR packed representation to represent a single state. For our classical planning benchmarks, we observe a median of 8 bytes, which fit into a single node. Trees with a single node result in no node sharing, but they add memory overhead per node, rendering tree databases redundant in this case. 

Figure~\ref{fig:fast-downward-packed-comparison} compares Hashset-Packed with the most promising dynamic tree database \dtreedbsswiss{} regarding time, state set memory usage, and compression ratio. In the leftmost figure, we observe that tree databases incur a small but constant runtime overhead, with no outliers. The center figure shows that the memory usage of \dtreedbsswiss{} is significantly higher in 719 tasks and lower in 390 tasks. This gap raises the question of whether there are cases where tree databases can outperform Hashset-Packed. The rightmost figure shows the state set compression ratio relative to the number of FDR state variables. The compression ratio increases almost linearly with the number of FDR variables, demonstrating the strength of tree databases on ground tasks with many FDR variables. For example, if the number of FDR state variables exceeds 30, the compression ratio is higher across all tasks except two.

We conclude that tree databases do not outperform Hashset-Packed in general. However, they can substantially reduce memory usage when the number of FDR state variables is large, a property that can be easily checked before running a search.

\begin{table}[t]
    \setlength{\cmidrulekern}{15pt}
    \centering
    \begin{tabular}{@{}ll@{\asep}rrrr@{}}
        & Domain         & \coverage & \outofmem & \outoftime & \memoryscore \\
        \midrule
 {\multirow{6}{*}{\rotatebox[origin=c]{90}{Classical}}}
        & Grounded-Hashset           & 1965          & 4420          & 96          & 1080          \\
        & Grounded-\dtreedbsswiss{}  & 2012          & 3959          & 510         & 1101          \\
        & Grounded-\dtreedbshashid{} & 2010          & 3993          & 478         & 1124          \\
        & Lifted-Hashset             & 2037          & 4182          & 262         & 1146          \\
        & Lifted-\dtreedbsswiss{}    & \textbf{2100} & 3809          & 572         & 1164          \\
        & Lifted-\dtreedbshashid{}   & 2094          & 3822          & 565         & \textbf{1188} \\
        \midrule
 {\multirow{6}{*}{\rotatebox[origin=c]{90}{Numeric}}}
        & Grounded-Hashset           & 216          & 341          & 18         & 88          \\
        & Grounded-\dtreedbsswiss{}  & \textbf{219} & 350          & 6          & 87          \\
        & Grounded-\dtreedbshashid{} & \textbf{219} & 350          & 6          & \textbf{90} \\
        & Lifted-Hashset             & 216          & 337          & 22         & 89          \\
        & Lifted-\dtreedbsswiss{}    & 218          & 339          & 18         & 88          \\
        & Lifted-\dtreedbshashid{}   & 218          & 340          & 17         & \textbf{90} \\
        \bottomrule
    \end{tabular}
    \caption{Results for the Mimir configurations on the 6495 classical and 380 numeric tasks.}
    \label{table:mimir-results-overview}
\end{table}

\begin{figure*}[t]
    \begin{tabular}{ccc}
        \includegraphics[width=0.31\textwidth]{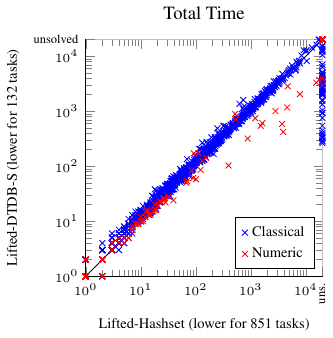} &
        \includegraphics[width=0.31\textwidth]{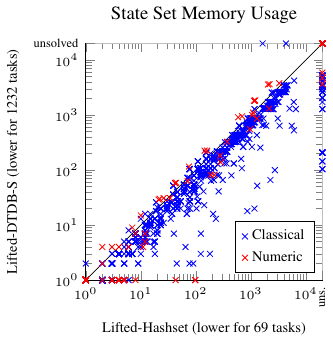} &
        \includegraphics[width=0.31\textwidth]{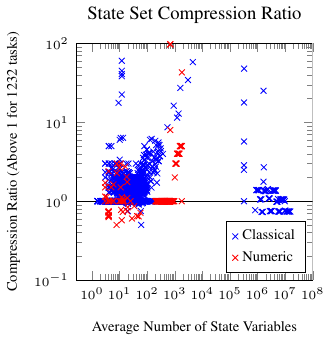}
    \end{tabular}
    \caption{Comparison of Lifted-Hashset and Lifted-\dtreedbsswiss{} (ours) showing the total time in seconds to solve a task (left), the memory usage for the state set in MiB (center), and the average number of state variables with the compression ratio between the configurations as described in Section 3.2 (right).}
    \label{fig:mimir-results-detail}
\end{figure*}

\section{Experimental Results for Mimir}

Table~\ref{table:mimir-results-overview} shows results for the Mimir configurations. We observe that all tree database configurations achieve strictly higher coverage than the hash set configurations. Furthermore, this change drastically decreases the number of times we run out of memory. Lastly, there is no significant overhead for resizing a hash ID map tree database, as the number of out-of-memory and out-of-time errors is almost identical to that of the off-the-shelf Swiss table variation. In contrast to \dtreedbsswiss{}, our hash ID map implementation causes significant memory spikes upon resizing, and explains why \dtreedbshashid{} does not achieve the highest coverage, despite achieving the best memory score of 1188.

Figure~\ref{fig:mimir-results-detail} shows a more detailed comparison between Lifted-\dtreedbsswiss{} and Lifted-Hashset.
The leftmost plot shows that tree databases come essentially for free, with negligible runtime overhead. The center plot reveals that memory usage rarely increases and often decreases significantly, i.e., lower memory usage for 2180 tasks and higher in only 84 tasks. The rightmost plot shows that the compression ratio approaches two orders of magnitude. For numeric planning, we often see reductions of at least one order of magnitude.

\section{Future Work}

\paragraph{Compact Hashing.} Compact hashing solves a restricted variant of the dictionary problem where keys are drawn from a universe $U = \{0,1,\ldots,u\}$ and the hash function $h$ is bijective \cite{cleary-ieeecomp1984}. In a hash set with $b$ buckets, e.g., with closed addressing \cite{koppl-et-al-fast2022}, rather than storing the entire key $x$ in bucket $i = h(x) \bmod b$, compact hashing stores only the quotient $q(x) = \lfloor h(x) / b \rfloor$. The original key can then be reconstructed as $h^{-1}(q(x) \cdot b + i)$, leveraging the bijectivity of $h$. This representation approaches the information-theoretic lower bound $\mathcal{B}(u, m) = \log_2\binom{u}{m} = m \log(u) - m \log(m) + \mathcal{O}(m)$ for storing a subset of $m$ keys from $U$. Compact hashing operates at the bit level, imposing additional runtime costs for reading and writing data. In the context of tree databases, keys correspond to nodes, i.e., pairs of $\wordsize$-bit integers determined by the node capacity, interpreted as a single $2w$-bit integer. Each node thus requires $w$ bits plus $O(1)$ bookkeeping bits for storage. Although compact hashing can reduce memory usage by roughly a factor of two compared to conventional methods, the associated bit-level access overhead makes it impractical for planning applications.

\paragraph{Tree Structures.} Early results indicate that the Unpacked-Tree becomes smaller when structuring the tree to minimize the number of potential entries. We achieved this with Huffman minimum redundancy codes, which is an optimal greedy strategy for structuring a tree with weighted leaf nodes such that the average weighted path from the leaf to the root is minimized \cite{huffman-1952}. The weight in this case is the number of potential entries, i.e., the size of the mutex groups in the FDR representation. This means that variables with larger groups, which contribute more to the total number of entries, can be placed closer to the root of the tree, thereby reducing the total representation size. 

\section{Conclusions}

We proposed dynamic variations of tree databases to encode sequences over an arbitrary finite alphabet. Our first variation uses off-the-shelf flat hash table implementations for simplicity, while our second variation uses hash ID maps, which require less memory per node. We applied dynamic tree databases to classical and numeric planning benchmarks and empirically compared them against three standard state representations: FDR, sparse propositional, and dense numeric. We observe compression ratios of several orders of magnitude in the sparse and numeric setting, often with negligible runtime overhead. Particularly in numeric benchmarks, dynamic tree databases often use less memory and solve more tasks than traditional representations.

\section*{Acknowledgements}

We would like to thank the anonymous reviewers and Elliot Gestrin for providing valuable feedback on an earlier draft of this work. This work was partially supported by the Wallenberg AI, Autonomous Systems and Software Program (WASP) funded by the Knut and Alice Wallenberg Foundation. The computations were enabled by resources provided by the National Academic Infrastructure for Supercomputing in Sweden (NAISS), partially funded by the Swedish Research Council through grant agreement no. 2022-06725.

\bibliography{abbrv,literatur,crossref,extra}

\end{document}